\documentclass[preprint]{elsarticle}  
\usepackage[utf8]{inputenc}
\usepackage{subcaption}
\usepackage{graphicx}
\usepackage[dvipsnames]{xcolor}
\usepackage{subcaption}

\usepackage{amsmath}
\usepackage{amssymb}

\usepackage{amsthm}

\usepackage[normalem]{ulem} 

\graphicspath{{.}{figures/}}		

\definecolor{usedBefore}{rgb}{0.7, 0.2, 0.4}
\definecolor{blau}{rgb}{0.0, 0.0, 0.6}
\definecolor{orange}{rgb}{0.79, 0.521, 0}

\definecolor{koenigsblau}{rgb}{0.01, 0.28, 1.0}
\definecolor{aoenglish}{rgb}{0.0, 0.5, 0.0}
\definecolor{bleudefrance}{rgb}{0.19, 0.55, 0.91}
\ifdefined\DEBUG
  \newcommand{\mf}[1]{{\color{orange}\textbf{MF}: #1}}
  \newcommand{\pe}[1]{{\color{bleudefrance} (\textbf{PE:} #1)}}

  \newcommand{\patrickdel}[1]{{\color{bleudefrance} \sout{#1}}}
  \newcommand{\bn}[1]{{\color{aoenglish} (\textbf{BN:} #1)}}
  
  \newcommand{\bndel}[1]{{\color{aoenglish} \sout{#1}}}
  \newcommand{\emmerix}[1]{{\color{koenigsblau} (\textbf{ME:} #1)}}
  
  \newcommand{\emmerixdel}[1]{{\color{koenigsblau} \sout{#1}}}
  \newcommand{\martina}[1]{{\color{purple} (\textbf{MF:} #1)}}
  
  \newcommand{\martinadel}[1]{{\color{purple} \sout{#1}}}
  \newcommand{\tb}[1]{{\color{magenta}\textbf{TB}: #1}}
  \newcommand{\tbdel}[1]{{\color{purple} \sout{\textbf{TB}: #1}}}
\else
  \newcommand{\mf}[1]{}
\newcommand{\pe}[1]{}
  \newcommand{\bn}[1]{}
  
  \newcommand{\bndel}[1]{}
  
  \newcommand{\patrickdel}[1]{}
  
  \newcommand{\emmerix}[1]{}
  
  \newcommand{\emmerixdel}[1]{}
  \newcommand{\martina}[1]{}
  
  \newcommand{\martinadel}[1]{}
  \newcommand{\tb}[1]{}
  \newcommand{\tbdel}[1]{}
\fi

\title{Optimally Weighted Ensembles of Regression Models: Exact Weight Optimization and Applications}
\author{Patrick Echtenbruck${}^*$}
\author{Martina Echtenbruck${}^*$}
\author{Joost Batenburg${}^*$}
\author{Thomas B{\"a}ck${}^*$}
\author{Michael Emmerich${}^*$}
\author{Boris Naujoks${}^+$}
\address{${}^*$LIACS, Leiden University, Niels Bohrweg 1, The Netherlands, \\
${}^+$IDE+A, TH Cologne, 51643 Cologne Gummersbach, Germany }

\begin{document}
\begin{abstract}
Automated model selection is often proposed to users to choose which machine learning model (or method) to apply to a given regression task. In this paper we show that combining different regression models can yield better results than selecting a single ('best') regression model, and outline an efficient method that obtains optimally weighted convex linear combination from a heterogeneous set of regression models. More specifically, in this paper a heuristic weight optimization, used in a preceding conference paper, is replaced by an exact optimization algorithm using convex quadratic programming. We prove convexity of the quadratic programming formulation for the straightforward formulation and for a formulation with weighted data points. The novel weight optimization is not only (more) exact but also more efficient. The methods we develop in this paper are implemented and made available via github-open source. They can be executed on commonly available hardware and offer a transparent and easy to interpret interface.  The results indicate that the approach outperforms model selection methods on a range of data sets, including data sets with mixed variable type from drug discovery applications.
\end{abstract}    
\maketitle
\section{Introduction}
A common task in machine learning is to build a regression (or surrogate) model of a black-box function based on a set of known evaluation results of that function at some points. The  regression model can be used to partially replace the original function, for instance in cases where the original function is expensive to evaluate or evaluations are difficult to obtain for other reasons.

Surrogate models play a significant role in modern optimization, prediction, modeling or simulation tools. In recent years various types of surrogate models have been proposed in the machine learning literature and integrated as options in machine learning libraries. It remains however difficult for users to select the right method for the given data-set and often this model selection problem is solved by experimenting with different models, looking at the training error. 

Automatic model selection methods relief the user from this task. It can be accomplished by, for instance, ranking models based on the cross-validation error on the training data set. However, recent research has shown that one can do better than merely selecting the best model from the ensemble by
combining several of the available regression models. Whereas many of these methods rely on overly simplistic assumptions (majority vote, averaging) or introduce a large amount of additional complexity (genetic programming, symbolic regression), the recently proposed \emph{optimally weighted model mixtures} provide a good compromise between simplicity, flexibility and integration of various regression models~\cite{friesebuilding}.
Moreover, it is a model-agnostic approach and can combine regression models of various types, such as for instance artificial neural networks, Gaussian process regression, piecewise linear regression, and random forests. This paper provides a significant extension of our earlier work on such model mixtures \cite{Frie20a}. 

In particular, this paper deals with ensembles of regression models that use a linear combination and offer an interpretable, generally applicable, and  efficient approach for combining models. Instead of simply selecting the prediction of the best model based on the cross-validation error, we propose to combine the model predictions of various regression models in an optimal way considering covariance information. In the formulation we restrict ourselves to convex linear model combinations, which are easy to interpret and can be configured efficiently. This article has three novel contributions:

\begin{itemize}
\item Firstly, we show how to adapt the approach to data sets where sample points are unevenly distributed. This is accomplished by a linear weighting scheme for data points based on point density that can be integrated into the exact solution method without a significant increase in complexity.

\item Secondly, we replace the heuristic optimization procedure used in a previous paper by an exact method, using a quadratic programming (QP) formulation that can be solved efficiently by standard QP solvers.

\item Thirdly, we apply and test the method on a range of data sets, including examples from drug discovery that require the combination of regression models for high dimensional functions with discrete variable types.
\end{itemize}

This article is structured as follows: In Section~\ref{preliminaries} we introduce the preliminaries needed for the theoretical part of the paper. In Section~\ref{relatedworks} we discuss previous and related works. In the Sections~\ref{clustering} to~\ref{quadratic} we develop the approach presented in this paper. Firstly, Section~\ref{clustering} discusses the relevance of clustering in the data and presents a method to effectively handle clustered data. Secondly, Section~\ref{n-aryEnsembles} discusses the need for large ensemble sets and illustrates that the search space for the optimal ensemble setup is convex, while Section~\ref{quadratic} gives the mathematical proof that the problem of finding optimal weights is convex 
and the error minimizing weights combination can directly be calculated by quadratic programming. In Section~\ref{Application to Drug Property Prediction} the adapted ensemble algorithm is tested on classical regression data sets as well as on different data sets from drug property prediction. The results of these experiments are presented and discussed in Section \ref{ResultsDiscussion}. Main aspects and findings of this paper are summarized and possible future works are discussed in Section~\ref{SummaryOutlook}.
\section{Preliminaries}
\label{preliminaries}


~\\
In machine learning, it is a common task to model a given objective function in order to classify unknown points or to predict promising parameter settings. In both cases, costly function evaluations on the original models are reduced by partly replacing them with fast approximate evaluations on surrogate models. By \textit{surrogate model}, we understand a function $\hat{f}: \mathbb{R}^d \rightarrow \mathbb{R}$ that is an approximation to the original function $f: \mathbb{R}^d \rightarrow \mathbb{R}$, learned from a finite set of evaluations of the original function.

The number of available modeling algorithms, all featuring different strengths and weaknesses, for the user to choose from is large. However, the choice of the model is crucial for the solution quality. Burnham et al. even state that the selection of the right surrogate model is the most crucial question in making statistical inferences~\cite{Burn2002a}.
To choose the best surrogate model for a given objective function, often expert knowledge about the surrogate model as well as the objective function is needed. But if no preliminary knowledge about the surrogate model or the objective function is available, it would be beneficial if an algorithm could learn which surrogate model suits best to a given problem.

Model selection approaches, which a priori train a set of models on training data and then choose the best surrogate model using statistical approaches, are already well established. However, recent results show that it can be beneficial to linearly combine several surrogate models into one better-performing ensemble model~\cite{friesebuilding,Bart16n, bates1969combination}. 
Further constraints on the ensemble coefficients lead to convex combinations, not to be confused with convex functions, and can be defined as follows. Given $s$ different surrogate models $\hat{f}_i: \mathbb{R}^d \rightarrow \mathbb{R},~ i=1, \dots, s$, $d$ the input dimension of the approximated functions and $\alpha_i$ the weights for the $i$-th surrogate model, a \textit{convex combination of models} (CCM) specifies an ensemble of surrogate models as follows:
\[\sum_{i=1}^s \alpha_i \hat{f}_i \text{\quad s.t. } \sum \alpha_i =1 \text{ and } \alpha_i \geq 0, \; i =1, \dots, s\]

Due to the summing up to unity constraint, the search space of model weights is an $(s-1)$-dimensional simplex. Special solutions, where only one model is used in the "ensemble", are located in the corners of the simplex.
To find the optimally weighted ensemble a minimization of the cross-validation error over the set of possible CCMs can be performed by searching over the simplex $\{\alpha \in \mathbb{R}^s| \sum_{i=1}^s \alpha_i =1, \alpha_i \geq 0  \}$.

As cross-validation error or fitness function resp., an adaptation of the root mean squared error (RMSE) is regularly used. The RMSE of the predictions of a single model is defined as $RMSE=\sqrt{\frac{1}{n}\sum_{i=1}^{n}{(y_i-\hat{y}_i)^2}}$, where $n$ is the number of predictions of this model. As a remark, it suffices to find the minimizer of the squared sum $\sum_{i=1}^{n}{(y_i-\hat{y}_i)^2}$ and it is equivalent to the minimizer of the $RMSE$.

\section{Related Work}
\label{relatedworks}
The general scenario that is investigated in our studies is the automated selection and optimal mixture of prediction or regression models. Ensemble methods have a long history and some early work has already been presented in the 1970ties \cite{bates1969combination} for time series predictions. 

The work presented in this article is a follow up and extension to an earlier conference paper \cite{friesebuilding}, where the idea of building sparse ensembles by mixtures or convex combination of heterogeneous models has been proposed and studied on a  benchmark set with multi-modal regression problems and a heuristic optimization with an (1$+$1)-ES is used to optimize the weights of the ensemble. In ~\cite{Frie20a} a review of other types of ensemble approaches is given. 


The approach in \cite{friesebuilding} was recently taken up in \cite{benitez2021sparse} in the context of predicting time series from the COVID-19 pandemic and an alternative penalizing term was introduced that is similar to the approach used in the LASSO regression model, where instead of models, variables are selected based on their marginal distribution. 

The new approach of our article extends the work by Friese et al.~\cite{friesebuilding} by the following major contributions. Firstly we show that finding optimal model mixtures is a convex quadratic optimization problem that can be solved to optimality, even if density of points is considered (weighted training set). Secondly, the approach is used on a broader data-set including discrete variables and data from the important application domain of molecular property prediction in drug discovery.

 Acar and Rais-Rohani~\cite{Acar2009a} proposed some adaptations of previously defined approaches of weighted sum ensembles for better generalization. Like in previous works they also required the weights to be positive and sum up to one as the only constraints on the weights. Building on the approach of Bishop et al.~\cite{Bish1995a}, they proposed to use k-fold cross-validation to allow for an evaluation of models that have per definition no error at the training points. For the works of Goel et al.~\cite{Goel07} they suggested to generalize the mean squared error\footnote{The GMSE refers to the MSE applied in a leave-one-out cross-validation process.}(GMSE) of the ensemble.



\section{Local Density Weighted Cross-Validation}  
\label{clustering}
The ensemble method, as first presented by Friese et al.~\cite{friesebuilding}, performed all experiments on mathematical test functions generated by Max-Set of Gaussian Landscape Generator \cite{Gallagher2006}. These test functions are trained on data-sets generated automatically using Latin Hypercube Design \cite{McKay1979}.
However, in real-world problem definitions, we cannot rely on the assumption that the available data is likely evenly spread. In sequential parameter optimization, it is even expected to build clusters of evaluated points at local optima.
Without taking this possible clustering of the points into account during cross-validation we risk to put too much emphasis on prediction errors in the area of clustered points, which leads to over-fitting in these areas.

One potential solution could be to exclude individual points situated in clustered areas from the set before evaluation to ensure an even distribution of the points in the data set before cross-validation. However, to do so it must be specified how many and which points will be excluded from the set. But, by doing so important information may be dropped from the data set and the surrogate models may perform better with the complete data-set.

Therefore, we propose to keep all data points and weight the squared error of points in denser areas to reduce the importance of these points. The given weights depend on the density of the direct neighborhood of each point and hence of the position of the regarded point. This way, the weighting occurs smoothly and without harsh steps between the weights of neighboring points. 
To this end, we use density weighted cross-validation applying a weighted Root Mean Square Error (wRMSE) as a quality indicator:\\ 
\begin{equation}
    wRMSE=\sqrt{\frac{1}{n}\sum_{i=1}^{n}{(\beta_i(y_i-\hat{y}_i))^2}}
\end{equation}


The weights ~$\beta_i\in[0,1]$ that are applied to the squared prediction errors $(y_i-\hat{y}_i)^2$ of their related predictions $\hat{y}$ at the positions $\mathbf{x}_i\in\mathbb{R}^d$, ~$i=1,\dots,n$, $n$ the number of predictions, are derived from the proximity of the point's nearest neighbors, and thus the density of the points direct neighborhood. For the calculation of this density, the $k$ nearest neighbors of each point are utilized. Per default, $k$ is set to $20$. 

These $k$ nearest neighbors are utilized to calculate the density $dens_i$ of a point $x_i$ as their median Euclidean distance to this point:
\tb{$\mathbf{x}_i$ again.}

\begin{equation}
    dens_i = \text{median}\left\{\sqrt{\sum_{j=1}^{d}{(\mathbf{x}_{ij}-\mathbf{x}_{lj})^2}} \, \left|\, l=1,\dots, k~\right.\right\}
\end{equation}

Density values that are exceeding the overall mean density are truncated to the mean value. This way, it is ensured that points not located in sparse areas get higher weight during cross-validation, and on the other hand points in highly crowded regions receive a lower weight to prevent over-representation of these regions in the global modelling of the response function.  The resulting density values at the point locations will be denoted with $dens_i$. 
\tbdel{"The resulting" means after truncation, but what do you mean by "receive a lower weight"? That is implicitly happening, right? However, I would not call them $dens_i$, too, as they are not identical.}
They are normalized to the $[0,1]$ range by computing $\beta _i=dens_i/\max\{dens_0,\dots ,dens_n\}$ in order to obtain the weights $\beta_i\in[0,1]$. 
\tb{I thought we had $n$ points/density values, now we have $n+1$?}
Here, the lower bound does not have to be considered since the density values $dens_i$ 
\tbdel{why now an index $j$?}
are calculated from distance values and thus are positive per definition. Also, it is not intended to force zero weight on points with the highest density.

Applying these weights to the prediction errors in the calculation of the RMSE during cross-validation ensures that all points with a neighborhood not denser than the mean density are considered with full weight and only points located in denser neighborhoods are weighted according to the density of their neighborhoods, while the weighting during the transition from sparser areas to clustered areas is smooth.



\begin{figure}[!htb]
  \centerline{\includegraphics[width=\textwidth]{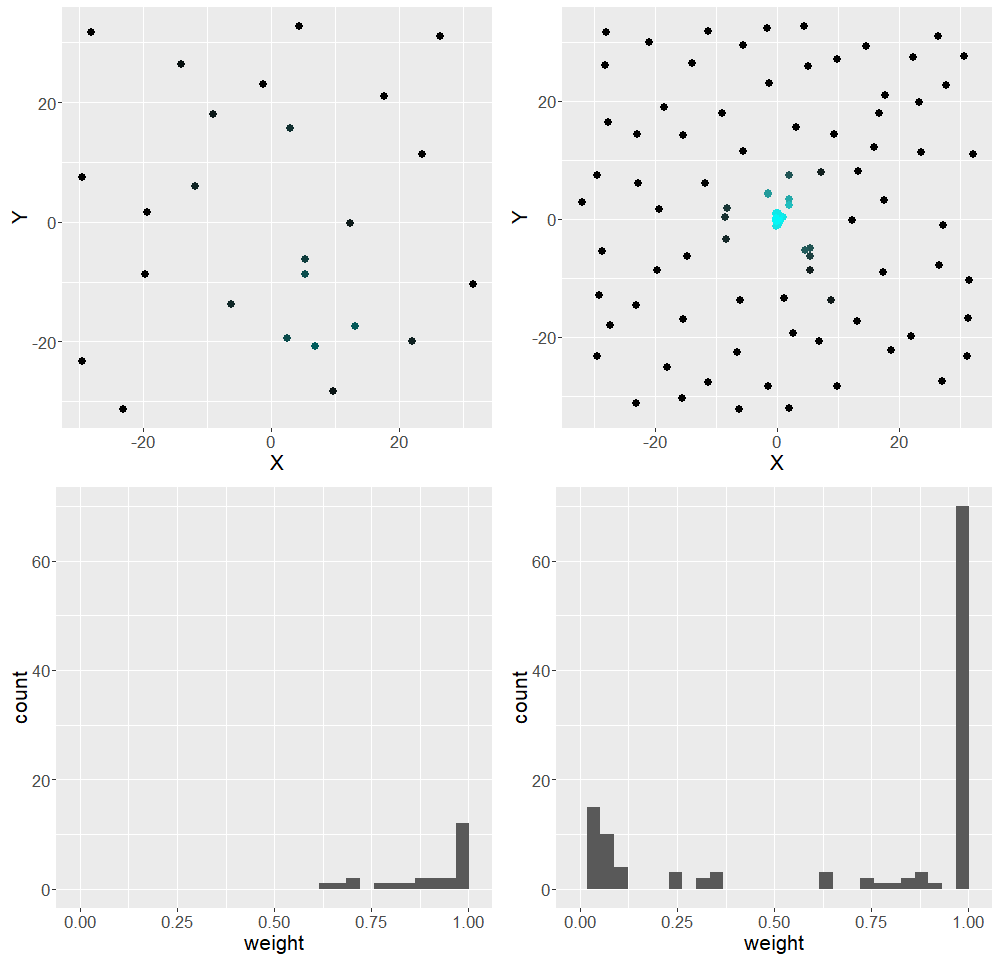}}
  \caption{The plots show the impact of the weighting procedure on the points at the beginning and the end of an optimization process on a 2D Ackley function. Points that are colored black are fully taken into account, and lighter blue points are weighted. The lower row shows the related distribution of weights used. By the light blue color of the cluster it can be seen that points near the cluster are rigorously weighted.
  \tbdel{what does "rigorously" mean here? How do we see that "near cluster" = "rigorously weighted"?}}
  \label{fig:weightingOfDatapoints}
\end{figure}


Figure \ref{fig:weightingOfDatapoints} illustrates the impact of the weighting on the points on an exemplary optimization situation. The example uses two states of the same optimization process that used a Latin Hypercube Design for the initial setup. Both plots on the left-hand side show the situation at the beginning of the optimization process, while the figures on the right-hand side illustrate the situation at the end of the optimization process. 
On the left-hand side, the points are evenly distributed over the search space, and, as the histogram in the lower row shows, most points get the full weight. Only a few points are lightly weighted.
On the right-hand side, the points have clustered around the local optima. Again the histogram shows that most points get the full weight, but some points, colored turquoise in the upper row, are considered with a higher weight.

\section{Ensemble Methods}\label{n-aryEnsembles}
\tbdel{What is $N$?} \pe{N-ary means "n entities", so ensembles consisting of n entities/models}
\tb{That was my point, so we should say $n$-ary Ensembles, if $N=n$?}\martina{The n/N was only a capital letter because it was the first letter of the title. But if this is deceptive we might better change it to n.}
In real-world regression or optimization problems and in black-box regression or optimization, the user often faces the problem that the main characteristics of the objective function are not known. Additionally, a wide range of potentially applicable surrogate models is available. Each of these models is based on specific model assumptions that define its strengths and weaknesses. If any, only very few users are acquainted with all of them. However, the right choice of a surrogate model for a specific problem definition is crucial for the quality of the regression or optimization result, respectively.
The same applies for the use of ensembles of surrogate models. Or even more so, since not only the most appropriate models need to be in the set but also weaker performing models, since they may still be able to compensate weaknesses of the stronger models. Hence, a large set of available surrogate models ought to be the best starting point for a best performing ensemble approach.

However, as the number of models grows, the complexity of finding the best model combination grows. Friese et al.~\cite{friesebuilding} applied a simple (1+1)-ES to search for the best model mixture on a set of three surrogate models, which showed promising results in reasonable time. Nevertheless, this approach deteriorates as the number of surrogate models, and with it, the number of dimensions of the search space grows. 

Echtenbruck \cite{Frie20a} showed that the algorithm completely failed to find a known good solution for three models in a set with additional ten models. 
Echtenbruck presented an incremental adaptation of the search strategy that searches for improving solutions in the different search dimensions successively. 
However, the improvement of the performance of the search algorithm goes along with an increased computation time for the search, linear in the number of surrogate models in the set. A more efficient approach is desirable.

Originally, the (1+1)-ES was chosen for its capability to perform well on convex functions as well as on non-convex or even multi-modal functions.
However, previously realised experiments suggested the assumption that the search space for the optimal weighting might be convex \cite{friesebuilding,Frie20a}. Also closer inspection of the applied error function conjectures that the optimization problem, specified by the MSE, is convex, noting that the minimizer of the MSE function coincides with the minimizer of the RMSE function.

In the following we show that the function is indeed convex. First we will motivate this heuristically, followed by a rigorous mathematical proof.

The RMSE is calculated based on an unordered set of prediction errors (deviation between prediction and true value) and is therefore independent from the structure of the underlying objective functions and the models used. Only the (weighted) mean of the squared deviations is to be minimized.
Considering the mixture of two models, the three following cases can be distinguished when a prediction for a single point is made:  
\begin{itemize}
\item[]{\bf Case 1}: both models deliver the same prediction. 
\item[]{\bf Case 2}: they deliver different predictions but both are either smaller or larger than the objective function value or, 
\item[]{\bf Case 3}: they deliver different predictions, but one model predicts larger and one smaller values than the objective function value.
\end{itemize}
Examples for these cases are shown in Figure \ref{fig:predWithErrors}. The Figure depicts predictions and errors of the ensembles that would result from a convex linear combination, using the mixture of models approach, as well as their resulting prediction errors. 

\begin{figure}[!htb]
  \centerline{\includegraphics[width=\textwidth]{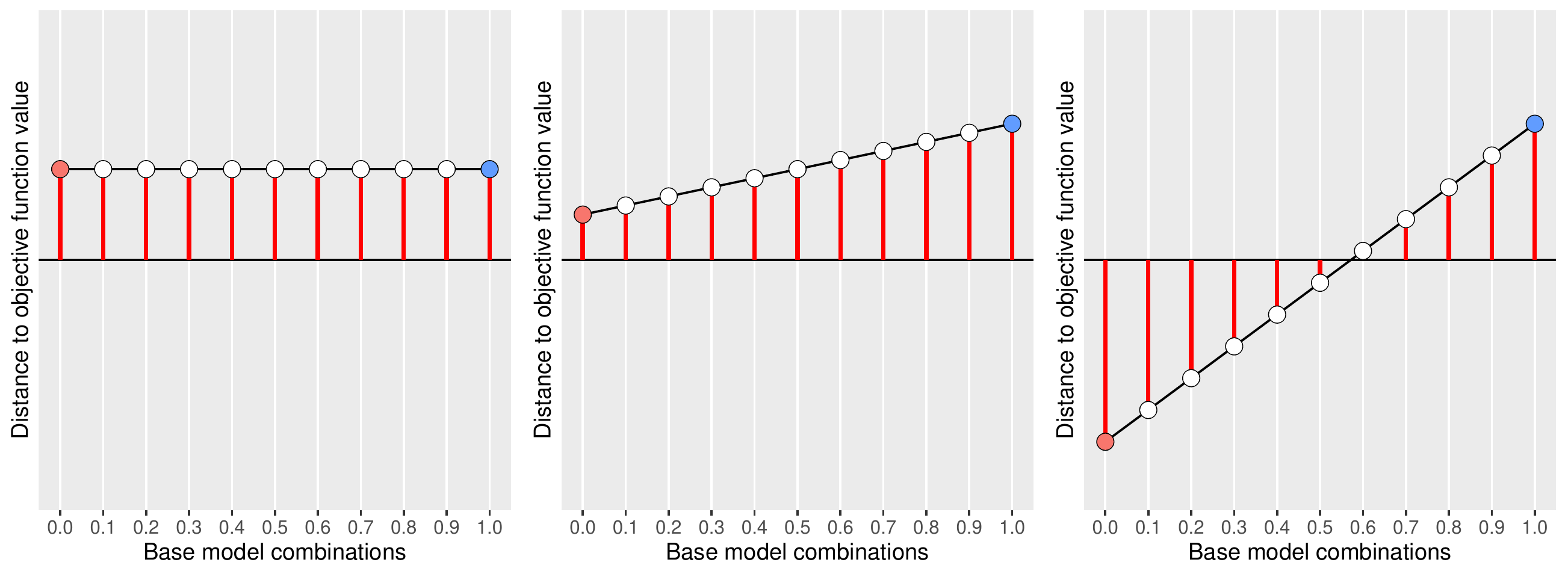}}
  \caption{The plots show the three different cases of prediction combinations of two base models for a single point. The horizontal line marks the distance zero to the objective function value \tbdel{I think it marks the distance zero to the objective function value, i.e., a perfect prediction.}\martina{and I would say the horizontal line marks the actual value of the considered point}, the dots mark the predictions of the different models. Here the red and the blue dot depict the base models, while the white dots depict the predictions of the ensembles that originate from the convex linear combination of the two base models. The red lines illustrate the prediction errors that result from the predictions of these models.
  }
  \label{fig:predWithErrors}
\end{figure}

For the calculation of the RMSE, the mean of the squared errors is considered. Since $\sqrt{x}$ is a bijective function and strictly monotonic for $x>0$, the square root can be omitted since it has no influence on the position of the minimizer (the point where the function obtains its minimum). Following previous thoughts, there are also three possible types of functions for the squared errors that have to be considered, these are shown in Figure \ref{fig:squaredErrors}.

\begin{figure}[!htb]
  \centerline{\includegraphics[width=\textwidth]{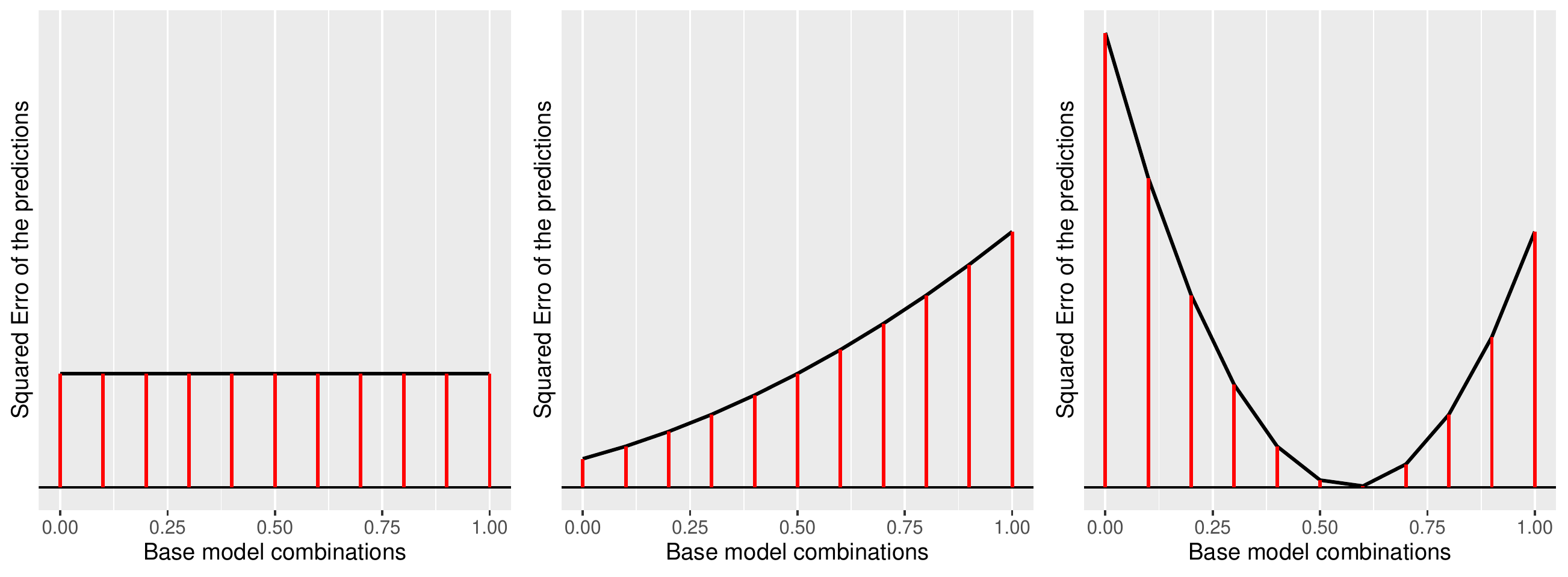}}
  \caption{The plots show the corresponding squared errors for the predictions shown in Figure \ref{fig:predWithErrors}, depicted as red lines. The black line depicts the resulting search space for the optimal convex combination of the two base models.
  }
  \label{fig:squaredErrors}
\end{figure}

If both models predict the same function value (Case 1) all predictions of the related convex combination models will be the same, and the resulting function of squared prediction errors is therefore also constant. In Case 2, where  both models predict different values, but both are either larger or smaller than the actual objective function value, the resulting function of squared errors will be quadratic, monotonic, and the optimum will be obtained at the boundary (meaning that in this case one of the models will be weighted with zero). 
In Case 3, as depicted in the right hand side plot of Figure \ref{fig:squaredErrors}, one model predicts larger and one smaller values than the actual objective function value. The resulting function of squared errors will be a quadratic function which obtains its minimum (of zero) in the interior of the interval, at the position where the convex combination of the base models would predict the actual function value.

The RMSE function that defines the search space for the optimal mixture of these two models is given by building the mean of $n$ squared-error functions, where $n$ is the number of training points. Adding a constant function type prediction of a new point (Case 1) to this mean would result in a compression of the original function that does not change the position of its minimizer. Adding an ascending or descending function type prediction (Case 2) to the mixture would result in a horizontal shift and if any in a horizontal compression of the mixture function. Adding a function, containing a minimizer at another position than contained in the mixture (Case 3), would result in a horizontal shift of the minimizer and if any in a horizontal dilation or compression of the existing mixture function. Figure \ref{fig:Mixtures} illustrates these cases. Algebraically, the superposition of a quadratic function and a quadratic function results in a quadratic function. 

\begin{figure}[!htb]
  \centerline{\includegraphics[width=\textwidth]{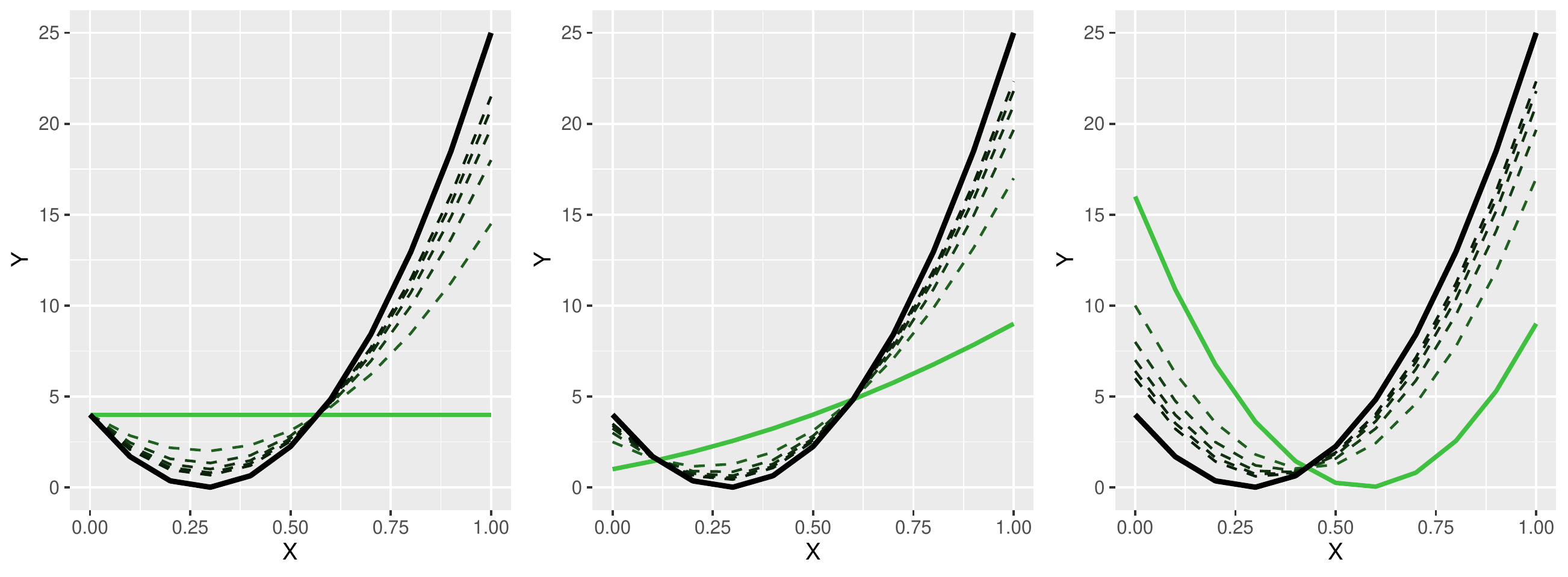}}
  \caption{In these plots the black lines represent a function that was created by building the mean of several functions. The green lines represent the functions from the three cases introduced in Figure \ref{fig:squaredErrors}. The dashed lines illustrate how the mean function would evolve if the function represented by the green line was added to the mean. Different examples are given in case the mean of the black line was built from 1,2,3,4 or 5 functions. Of course the influence of the added function would decrease with a higher number of already contained functions.}
  \label{fig:Mixtures}
\end{figure}

These thoughts are so far restricted to the combination of two models. However, any line segment in the search space is bounded by two end points that in itself are model-mixtures and thus models. Therefore also on this line segment the function is convex and quadratic, as in the case of two models discussed above. 

A mathematical proof, that the regarded search space is actually convex, is given in the following chapter.
\section{Exact Quadratic Programming Method}
\label{quadratic}
Let $\hat{f}_1,\ldots,\hat{f}_s: \mathbb{R}^d \to \mathbb{R}$ be the set of surrogate models. 

Let $\mathbf{x}_1,\ldots,\mathbf{x}_n \in \mathbb{R}^d$ be the set of sample points where the models are evaluated. 
We denote the vector of coefficients for the linear combination of surrogate models by 
$\mathbf{\alpha} \in \mathbb{R}^{s}$.
Define $\Omega = \{\mathbf{\alpha} \in \mathbb{R}^s | \sum_{i=1}^s \alpha_i=1, \alpha_i \geq 0 \; (i=1,\ldots,s)\}$, the set of valid weight vectors. Note that $\Omega$ is defined as an intersection of convex sets (one halfspace for each nonnegativity constraint and a hyperplane for the summation constraint) and is therefore convex. 

Define $A = (a_{ij}) \in \mathbb{R}^{s{\times}n}$ by $a_{ij} = \hat{f}_{j}(x_i)$. The matrix $A$ contains the evaluations of the surrogate models 
\tbdel{so it should be $\hat{f}_j$}
in the sample points, where each column of $A$ is an $s$-dimensional vector containing the evaluations for one of the surrogate models.

For a given vector of coefficients $\mathbf{\alpha}$ and vector $\mathbf{y}\in\mathbb{R}^n$ of desired outcomes, define 
\begin{equation}
\mathit{RMSE}(\mathbf{\alpha}, \mathbf{y}) = \sqrt{\frac{1}{n}||A\mathbf{\alpha}-\mathbf{y}||^2}
\label{eq:unweighted}
\end{equation}

Note that this definition is equivalent to the RMSE definition given in Section \ref{preliminaries}, except for a change in notation.

Consider the following problem:
\begin{equation}
    \mathbf{minimize}_{\alpha\in\Omega}\; \mathit{RMSE}(\mathbf{\alpha}, \mathbf{y})
    \label{min_rmse_eq}
\end{equation}
Or, equivalently, define:
\begin{equation}
\mathbf{\alpha}^* = \underset{\mathbf{\alpha} \in \Omega}{\arg \min} \;
    \mathit{RMSE}(\mathbf{\alpha}, \mathbf{y})
    \label{min_rmse_eq2}
\end{equation}

In the following we will show that the above problem is a convex quadratic programming problem and therefore it has a unique solution and can be solved with efficient standard solvers.

\newtheorem{theorem}{Theorem}
\begin{theorem}
The optimization problem stated in Equation \ref{min_rmse_eq} and \ref{min_rmse_eq2} is a convex quadratic problem.
\end{theorem}
\begin{proof}
As $n$ is fixed and $\sqrt{x}$ is a strictly monotonic function for $x \geq 0$, this minimization problem is equivalent to minimizing $||A\mathbf{\alpha}-\mathbf{y}||^2$. More specifically, for any $\mathbf{\alpha}_1, \mathbf{\alpha}_2\in\mathbb{R}^n$, we have that 
$\mathit{RMSE}(\mathbf{\alpha}_1,\mathbf{y}) \geq \mathit{RMSE}(\mathbf{\alpha}_2,\mathbf{y})$ if and only if $||A\mathbf{\alpha}_1-\mathbf{y}||^2 \geq ||A\mathbf{\alpha}_2-\mathbf{y}||^2$.
We therefore need to solve:
\begin{equation}
    \mathbf{minimize}_{\alpha\in\Omega}\; ||\mathbf{A}\mathbf{\alpha}-\mathbf{y}||^2
    \label{least_squares_eq}
\end{equation}

Define $Q = A^{T}A$ 
and $\mathbf{c}=A^{T}\mathbf{y}$. Then Problem \ref{least_squares_eq} can be written as a \emph{quadratic programming} problem in the following form:
\begin{equation}
  \mathbf{minimize}_{\mathbf{\alpha}\in\Omega}\; \left(\frac{1}{2}\mathbf{\alpha}^{T}Q\mathbf{\alpha} + \mathbf{c}^{T}\mathbf{\alpha}\right).
  \label{quadratic_programming_eq}
\end{equation}
Problem \ref{quadratic_programming_eq} is \emph{convex} as $Q$ is a positive semidefinite matrix, and the feasible set is the convex set $\Omega$.
\end{proof}
A variation of the problem, weighted RMSE ($w$RMSE) minimization, is based on the following definition:
\tbdel{The previous two sentences should be combined, e.g. "A variation of the problem, weighted RMSE ($w$RMSE) minimization,  is based on the following definition:"}
%
\begin{equation}
    wRMSE=\sqrt{\frac{1}{n}\sum_{i=1}^{n}{(\beta_i(y_i-\hat{y}_i))^2}} 
    \label{wrmse_eq}
\end{equation}

Like in the unweighted case it can be easily shown that this is a convex quadratic problem.
\begin{theorem}
Given a set of non-negative weights $\beta_i$, $i = 1,\dots, n$, 
the problem 
\begin{equation}
\mathbf{\alpha}^* = \underset{\mathbf{\alpha} \in \Omega}{\arg \min} \;
    \mathit{wRMSE}(\mathbf{\alpha}, \mathbf{y})
    \label{min_rmse_eq3}
\end{equation}
is a convex quadratic problem.
\end{theorem}

\begin{proof}
To show this, the formulation (of 
Eq.~(\ref{least_squares_eq}) 
can be modified in a straightforward way. Each row $i$ of the matrix $A$ (and each entry of the vector $\mathbf{y}$) will have weight $\beta_{i}$. 
All these weights jointly form the vector $\beta$. We then multiply in Eq. (\ref{least_squares_eq}) the rows of $A$ by $\beta$ (replacing $A$ by $\tilde{A} = \beta A$), and also the corresponding entries of y (replacing them by $\tilde{y} = \beta y$). We can then continue with the same reasoning as before starting from Eq.~(\ref{quadratic_programming_eq}) , but with a new matrix $\tilde{A}$ and vector $\tilde{y}$.
\end{proof}
For positive definite Q, the ellipsoid method solves the problem in (weakly) polynomial time~\cite{kozlov1980}.

To depict the differences between results received by exactly solving the above QP problem and an heuristic method, these methods are applied to the problem of finding the optimal convex combination of three base models on a given objective function. The experiment setup relies on the experiment setup initially introduced in \cite{friesebuilding}: We use the Max-Set of Gaussian Landscape Generator (MSG)~\cite{friesebuilding,Gall06a}, to generate a four dimensional objective function using 160 Gaussian process realizations, and Kriging surrogate models using three different kernels: gaussian, spline and exponential, following the definitions of~\cite{Loph02a}, as base models. To generate the experiment data, a Latin Hypercube Design of 160 points is evaluated on the objective function. The three base models are then evaluated on these data using a Leave-One-Out Cross-Validation to obtain 160 model predictions on the objective function for each model. For the heuristic method we used a simple (1+1)-Evolution Strategy with 1/5th success rule~\cite{Beye02a}, also as proposed by~\cite{friesebuilding}.

\begin{figure} 
    \centering
    \begin{subfigure}[b]{0.48\textwidth}
        \includegraphics[width=\textwidth]{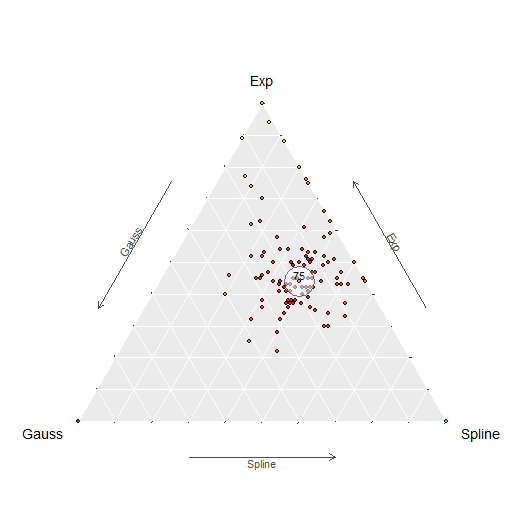}
        \caption{Search trajectory and optimal solution found by an heuristic method ((1+1)-Evolutionary Strategy)}
        \label{fig:gull}
    \end{subfigure}
    ~
    \begin{subfigure}[b]{0.48\textwidth}
        \includegraphics[width=\textwidth]{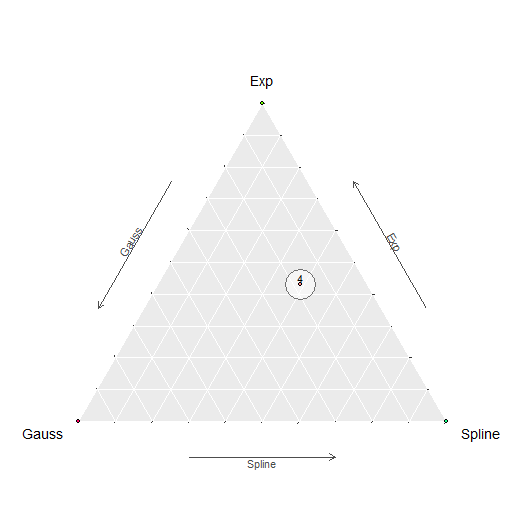}
        \caption{Optimal solution found by Quadratic Programming}
        \label{fig:tiger}
    \end{subfigure}
    \caption{All considered individuals are shown as points, the optimal solution is marked with a white circle. The Quadratic Programming solution is located close to the optimal solution found by the heuristic method.}
    \label{fig:TernaryComparison}
\end{figure}

Figure~\ref{fig:TernaryComparison} shows the results of these two weight optimizations in a ternary mixture diagram. The results of the $(1+1)$-ES and the QP method are similar, but due to its stochastic nature the $(1+1)$-ES is not guaranteed to converge, why it is preferred to use the exact solver, which has now been made available.

\section{Application}\label{Application to Drug Property Prediction}
A real world application of the optimally weighted ensemble approach is the field of drug discovery. 
\textit{De novo} drug discovery mainly targets the prediction of bio-activity. This way it can be determined whether a component is active or inactive on a specific target protein. From a chemist's perspective developing new drugs is a complex and costly task. The reason for that is that a lot of compounds fail compared to the small number of successful candidates.
There are two main goals in \emph{in silico} classification that are mainly considered. The prediction of activity on targets to identify compounds of high efficacy, and the prediction of activity on off-targets to avoid or exclude compounds that likely have unwanted side-effects. \emph{In silico} drug discovery can provide promising components in reasonable time. This procedure can minimize drug discovery costs by pre-selecting components that chemists need to test in "wet lab" experiments and exclude less promising components beforehand.

In our experiments drug data is represented as \emph{functional class fingerprints (FCFP)} \cite{Roge2010a} which is an established way of representing molecule data in a computer readable form and also retain comparability between molecules.

Promising results in combination with multiobjective optimization for classification are available \cite{Echt2019a}. Furthermore, optimally weighted ensembles were successfully applied to this kind of data \cite{Echt2021a}. 
Ensembles of heterogeneous models lead to better results than the single model approach.

The dataset is derived from the publicly available ChEMBL\footnote{Chemical database by EMBL-EBI; European Bioinformatics Institute at the European Molecular Biology Laboratory.} database. A tool written for this purpose gets all molecules with a known activity for a specific target protein from the ChEMBL API \cite{Echt2021a}. This data is converted to a dataset containing FCFPs. The activity of a molecule is defined by a proprietary continuous pChembl value~\cite{Bento2014a}.

The dataset we analysed further has the ChEMBL Id \emph{CHEMBL4159} and has the \emph{Endoplasmic reticulum-associated amyloid beta-peptide-binding protein} as a target. This target was mainly chosen because of the large count ($>$20,000) of known molecules that have an activity on this target.

To determine the influence of clustering and therefore reducing the weight of single points in cluster-regions we analysed the weights in detail. The plots in Figure \ref{fig:weightsCHEMBL4159} show the distribution of different weights in our \emph{CHEMBL4159} dataset. We have a total of 18\,452 points, 9\,572 of these  have a weight of 1. Since the histogram clearly shows that the weights tend towards a value of 1 and almost no values are below 0.50, there is not much clustering inside the CHEMBL4159 dataset. To further check this we generated a test-set with maximized coverage of the data points. This lead to almost identical results. The plots in Figure \ref{fig:weightsCHEMBL4159dist} show the distribution of different weights in our CHEMBL4159 dataset with a maximized distribution of data in the test set. We have a total of 18\,452 point of which 9\,586 have a weight of 1.

This findings are also reflected in the ROC (Receiver Operating Characteristics) plots in Figures \ref{fig:ROC_4159_mean_opt_RMSE} and \ref{fig:ROC_4159_mean_opt_RMSE_weighted}. The plots in Figure \ref{fig:ROC_4159_mean_opt_RMSE} show the mean result of unweighted RMSEs and the plots in Figure \ref{fig:ROC_4159_mean_opt_RMSE_weighted} show the mean result of weighted RMSEs. The results are almost identical and the optimal point defined by the Youden-Index J or Youden’s J statistic which was
originally proposed in \cite{Youd1950a} is nearly in the same spot. This results from the lack of clustering inside the CHEMBL4159 dataset where a weighting of point is obviously not necessary. This shows that the benefit of using the weighted RMSE is negligible when the data is distributed homogeneously. The weighted RMSE approach does not guarantee better results and mainly depends on the existence of clusters in the data. However, since it is generally not known beforehand how the data is distributed, the wRMSE should be the means of choice.

Also, experiments were conducted to determine if the quadratic programming approach leads to equally good or better results. The plots in Figure \ref{fig:ROC_4159_mean_opt_RMSE} and Figure \ref{fig:ROC_4159_mean_opt_RMSE_weighted} show these results. There is also the optimal Youden Index point marked. The results are similar to the ensemble results from the ES approach. What can be deduced from the plots the optimal Youden Index is almost identical with the benefit of a hugely improved runtime.

As described in Echtenbruck et al.~\cite{Echt2021a}, the mixture of models could indeed improve the results. One model that was used as part of the ensemble is an optimized ranger/random forest. This optimized model contributed hugely to the overall result with around 90\%, but the overall result could be improved with an added Lasso model. A Ridge model was also used, but could not contribute in a meaningful manner. It was excluded in 9 out of 10 experiments. With standard models that were not optimized for the problem beforehand, the ensemble has a more even distribution of contributing models. These results are confirmed when using the quadratic programming approach instead of an (1+1)-ES. Nevertheless, better results could be achieved if  all the contributing models are optimized beforehand.

\begin{figure}
    \centering
    \begin{subfigure}[b]{0.48\textwidth}
        \includegraphics[width=\textwidth]{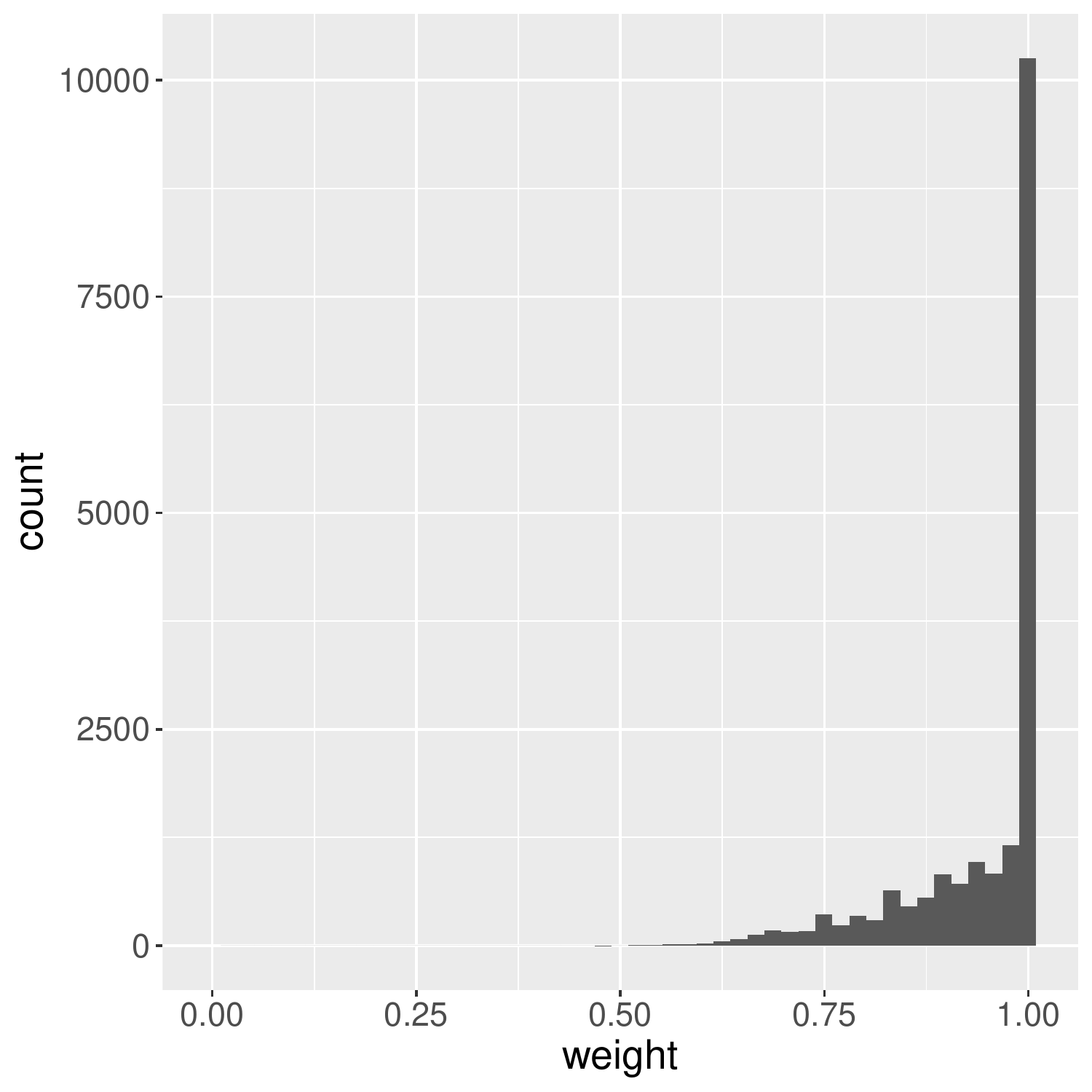}
        \caption{Weights applied to the points of the CHEMBL4159 dataset. Nearly 10\,000 points are weighted.}
    \end{subfigure}
    ~
    \begin{subfigure}[b]{0.48\textwidth}
        \includegraphics[width=\textwidth]{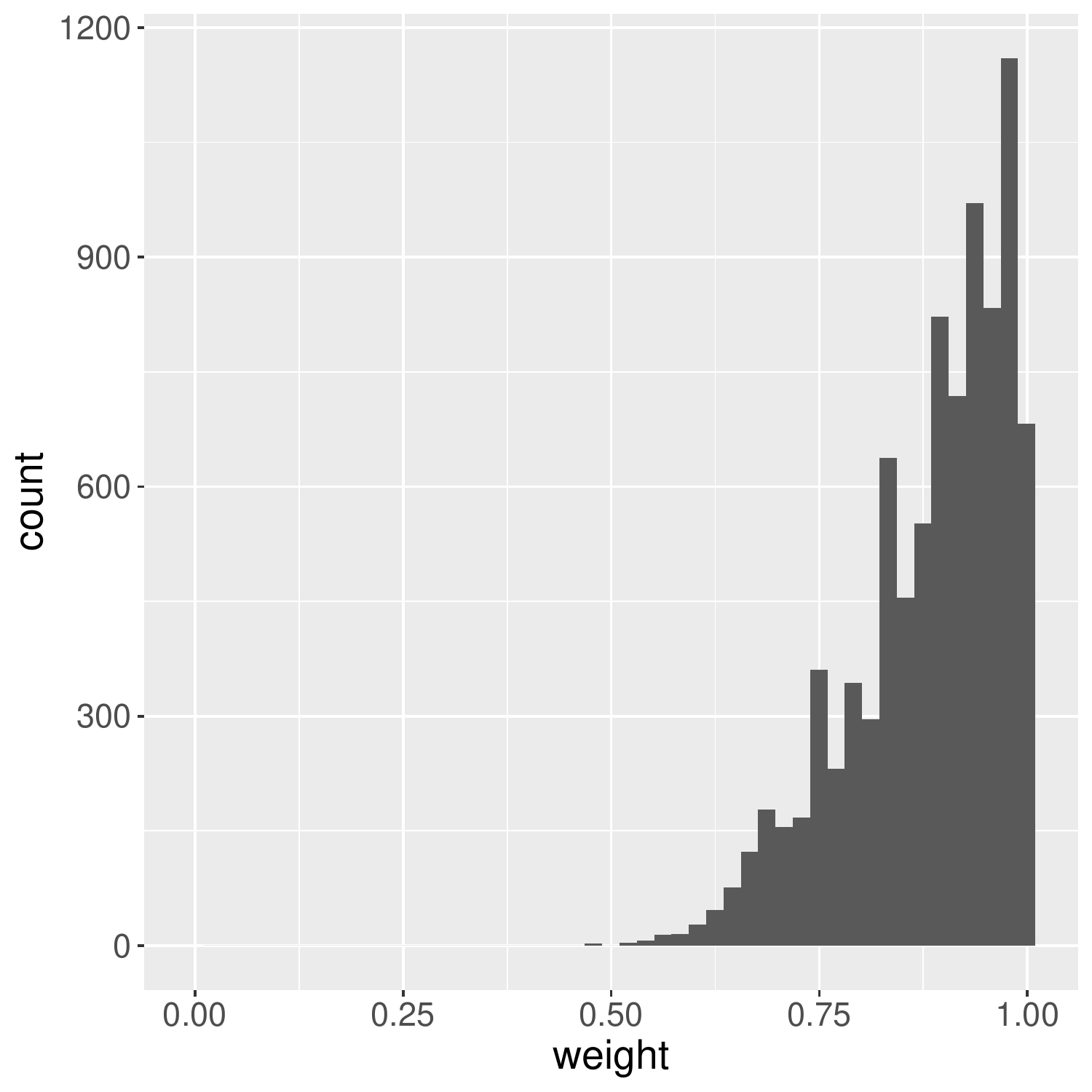}
        \caption{Weights applied to the points of the CHEMBL4159 dataset, points with weight=1 excluded.}
    \end{subfigure}
    \caption{Histogram of the applied weights on the CHEMBL4159 dataset.}\label{fig:weightsCHEMBL4159}
\end{figure}

\begin{figure}
    \centering
    \begin{subfigure}[b]{0.48\textwidth}
        \includegraphics[width=\textwidth]{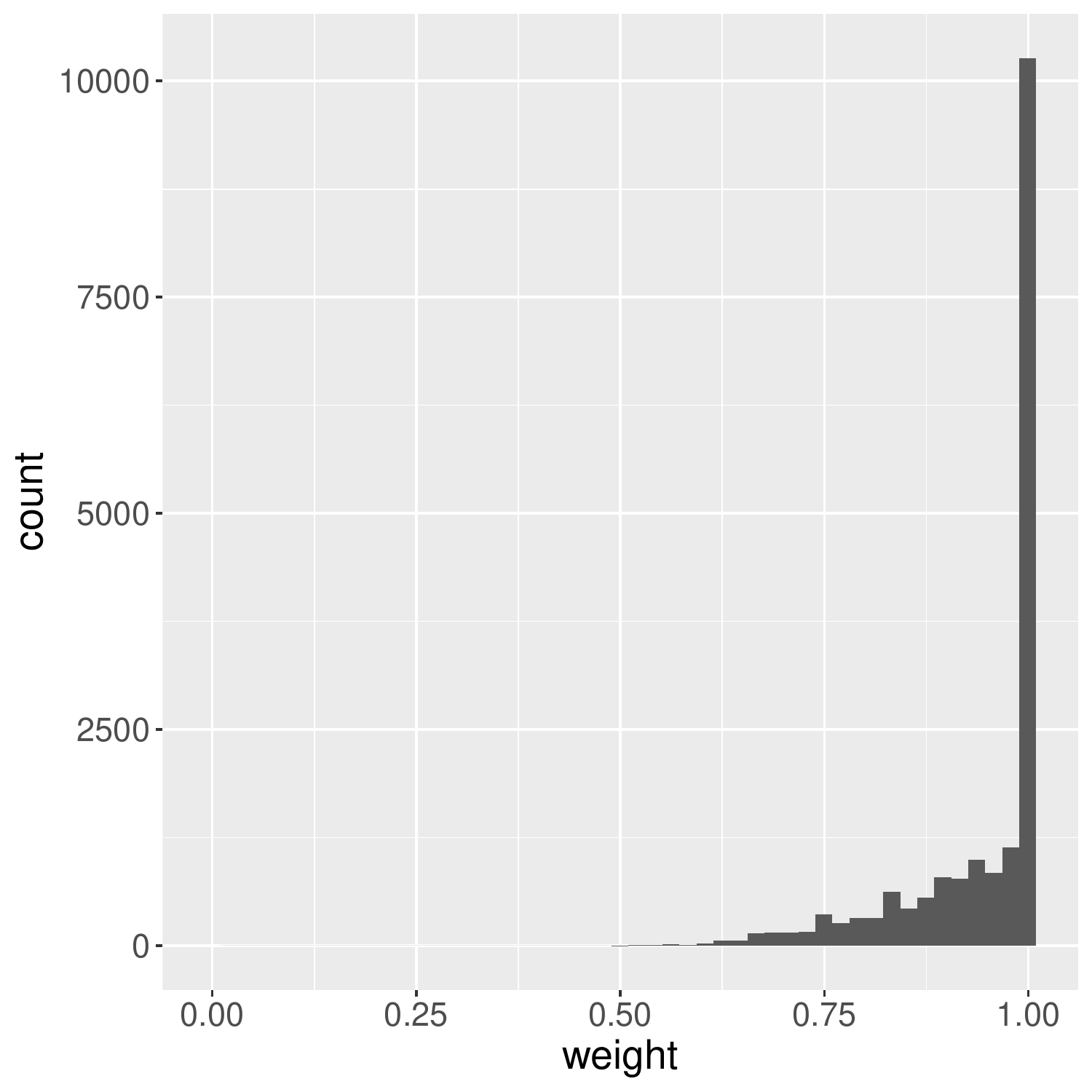}
        \caption{Weights applied to the points of the CHEMBL4159 dataset with max coverage inside the dataset. Nearly 10\,000 points are weighted.}
    \end{subfigure}
    ~
    \begin{subfigure}[b]{0.48\textwidth}
        \includegraphics[width=\textwidth]{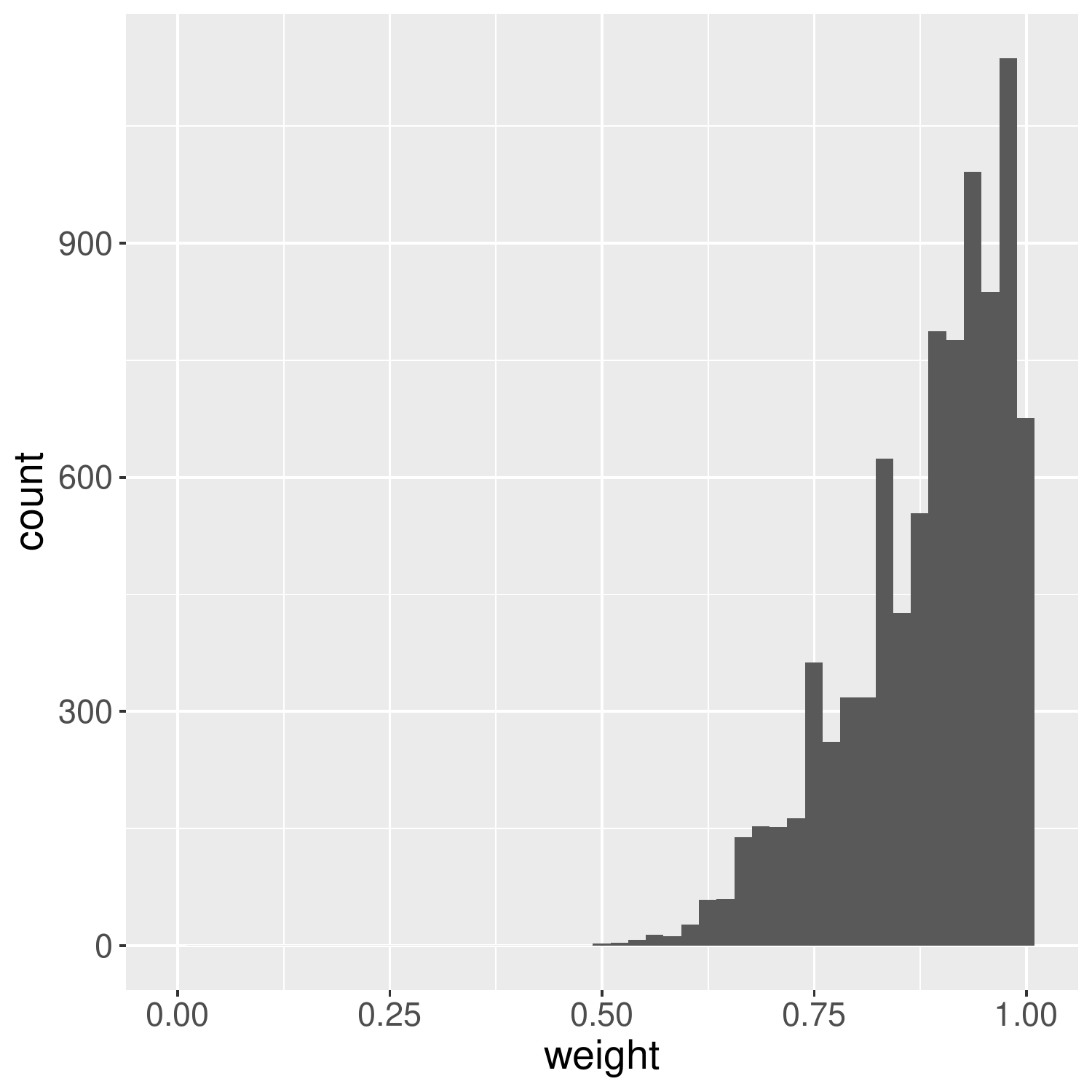}
        \caption{Weights applied to the points of the CHEMBL4159 dataset with optimized distribution, points with weight=1 excluded.}
    \end{subfigure}
    \caption{Histogram of the applied weights on the CHEMBL4159 dataset with max coverage inside the dataset.}\label{fig:weightsCHEMBL4159dist}
\end{figure}

\begin{figure}
    \centering
    \begin{subfigure}[b]{0.48\textwidth}
        \includegraphics[width=\textwidth]{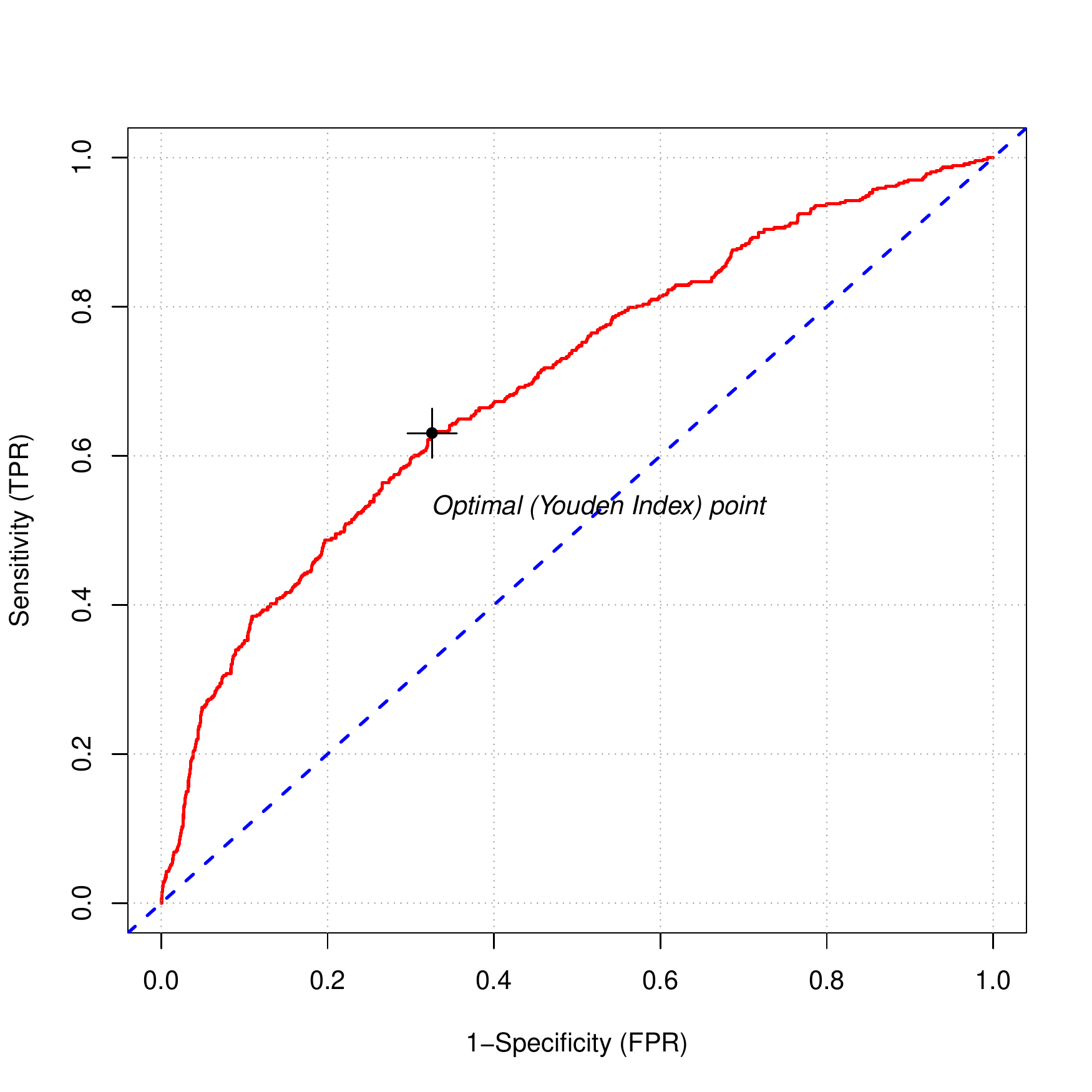}
        \caption{Mean curve calculated by ES with unweighted RMSE and Youden Index marked}
    \end{subfigure}
    ~
    \begin{subfigure}[b]{0.48\textwidth}
        \includegraphics[width=\textwidth]{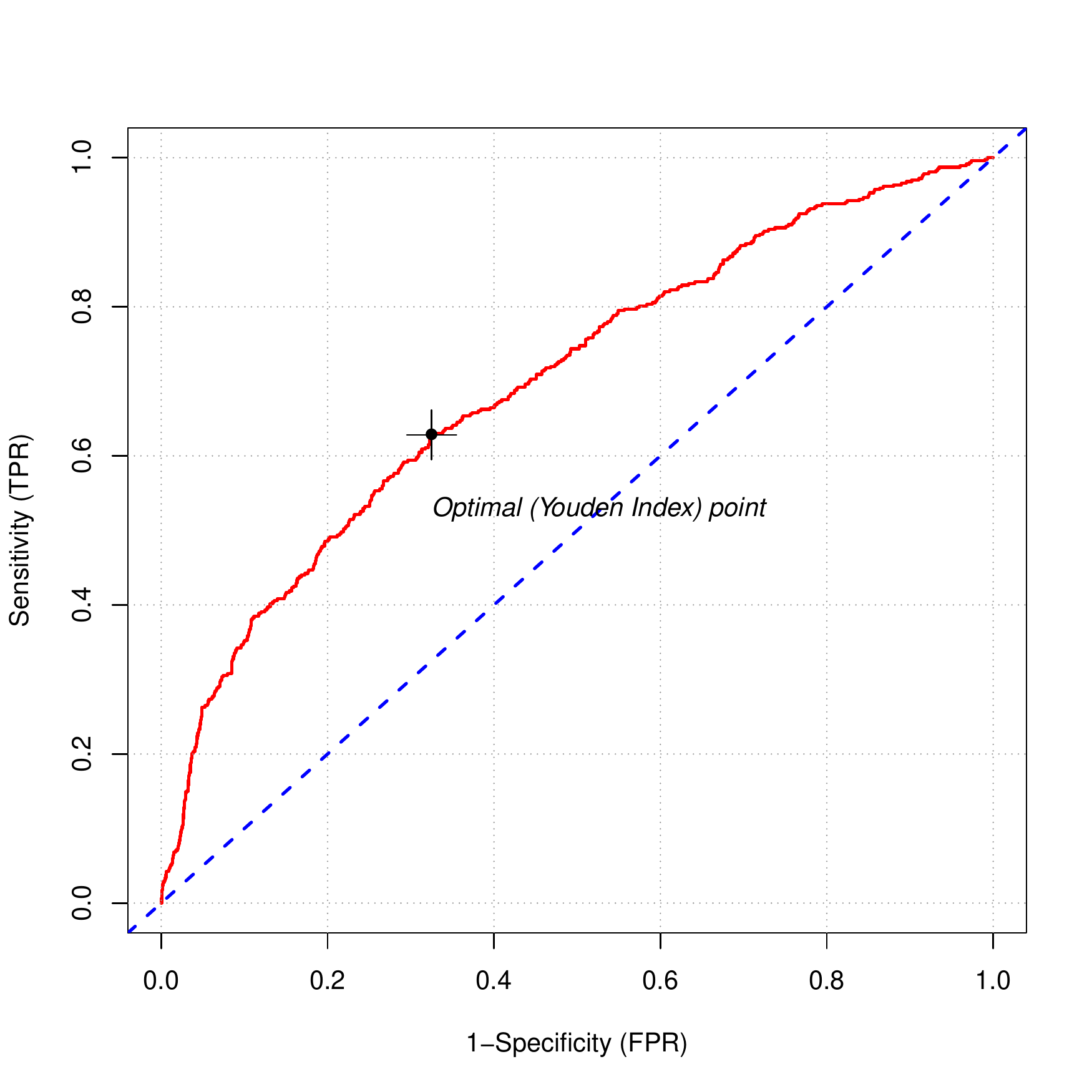}
        \caption{Mean curve calculated by QP with unweighted RMSE and Youden Index marked}
    \end{subfigure}
    \caption{The mean result of all 10 experiments with Youden Index marked (unweighted RMSE, max coverage in test set).\tbdel{would it make sense to put these into one figure, so one can see how close they are? Also for the next figure?}}\label{fig:ROC_4159_mean_opt_RMSE}
\end{figure}

\begin{figure}
    \centering
    \begin{subfigure}[b]{0.48\textwidth}
        \includegraphics[width=\textwidth]{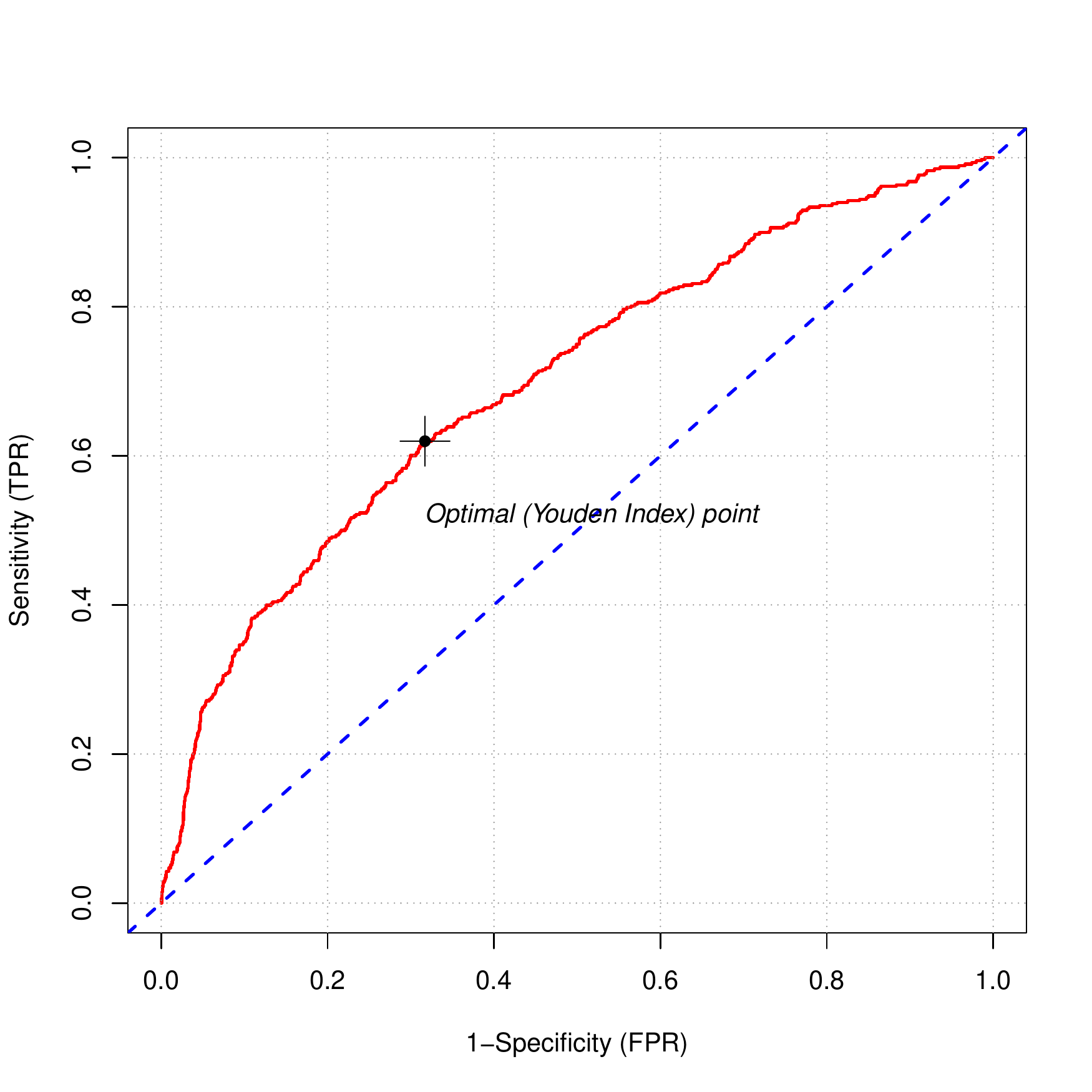}
        \caption{Mean curve calculated by ES with weighted RMSE and Youden Index marked}
    \end{subfigure}
    ~
    \begin{subfigure}[b]{0.48\textwidth}
        \includegraphics[width=\textwidth]{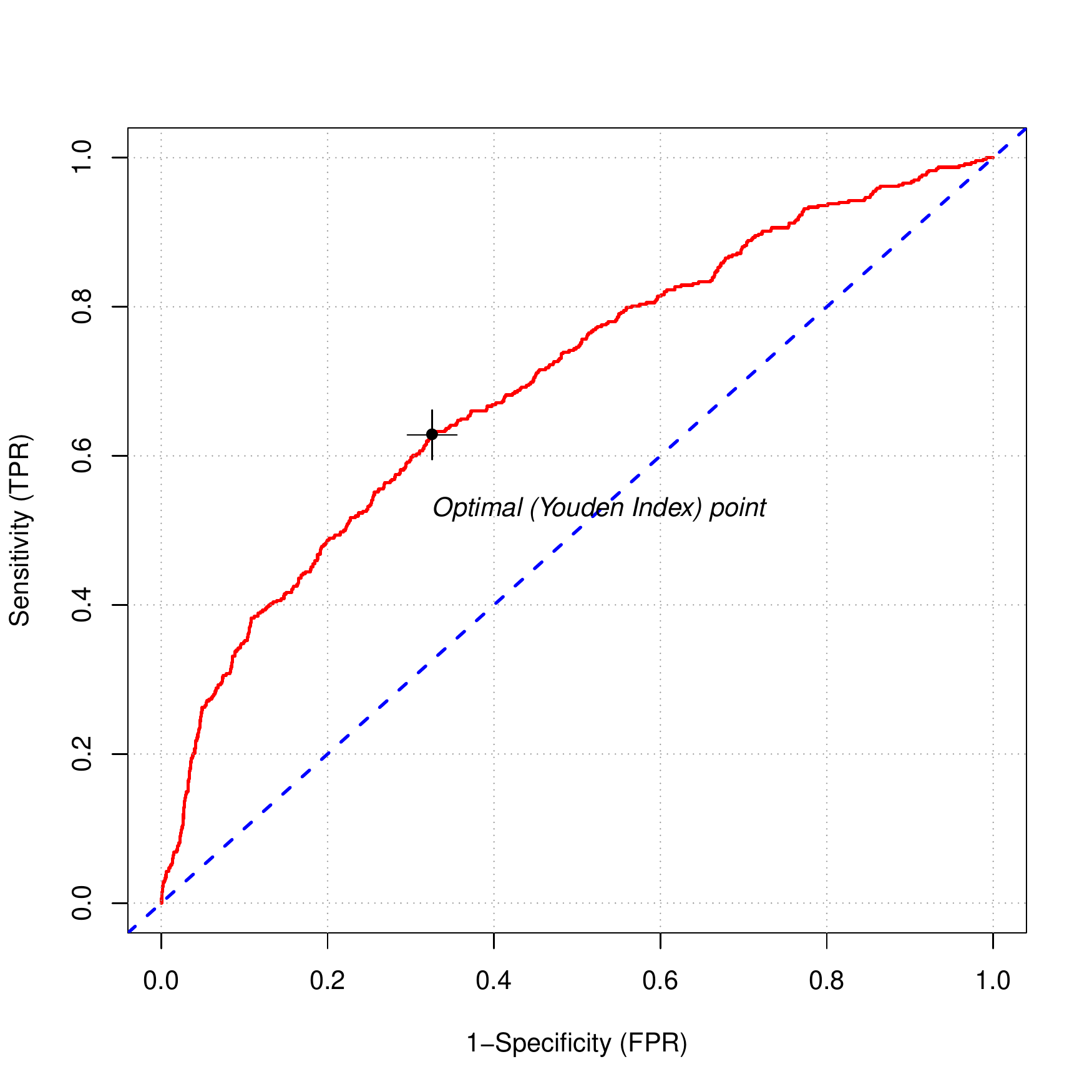}
        \caption{Mean curve calculated by QP with weighted RMSE and Youden Index marked}
    \end{subfigure}
    \caption{The mean result of all 10 experiments with Youden Index marked (weighted RMSE, max coverage in test set).}\label{fig:ROC_4159_mean_opt_RMSE_weighted}
\end{figure}

\begin{figure}
    \centering
    \begin{subfigure}[b]{0.48\textwidth}
        \includegraphics[width=\textwidth]{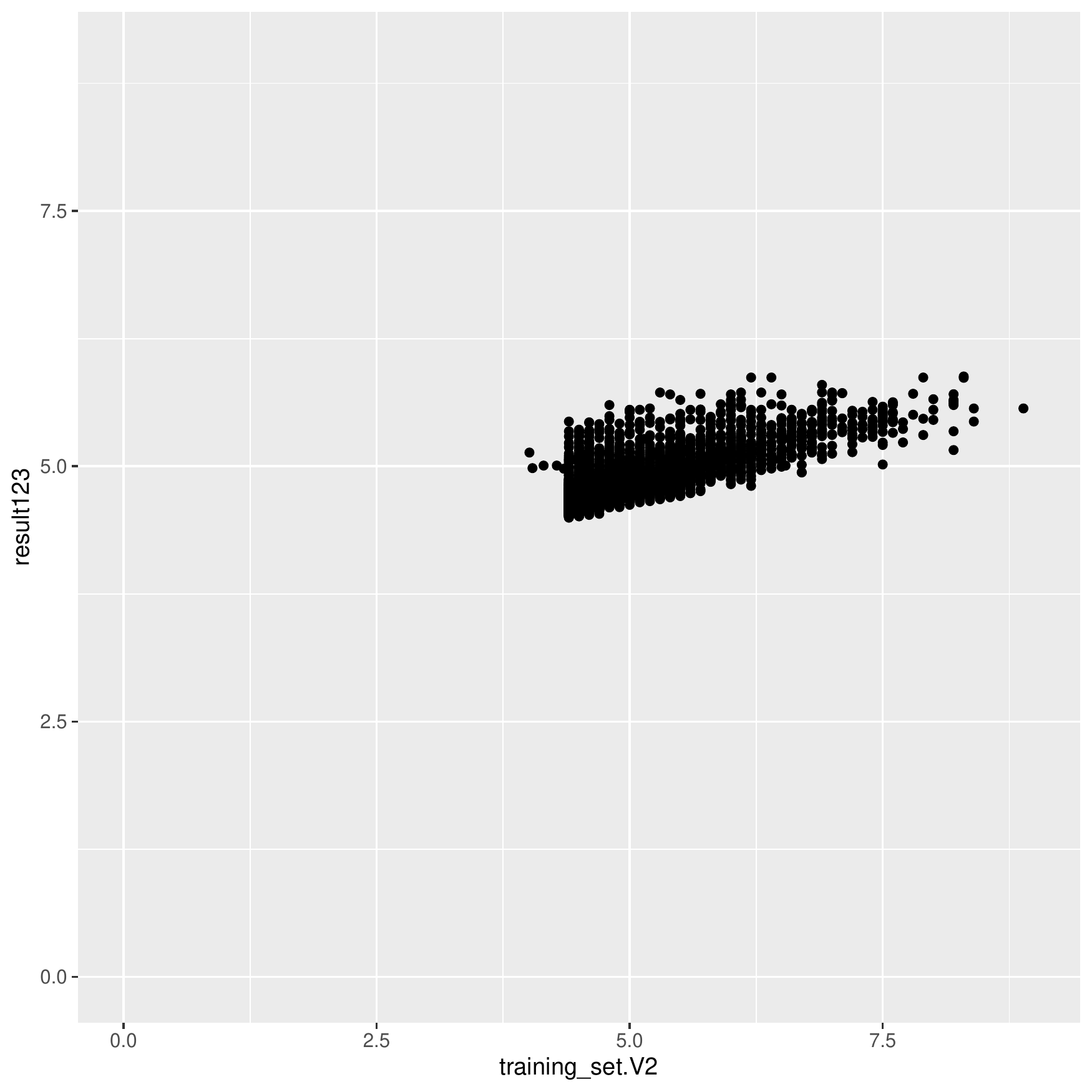}
    \end{subfigure}
    ~
    \begin{subfigure}[b]{0.48\textwidth}
        \includegraphics[width=\textwidth]{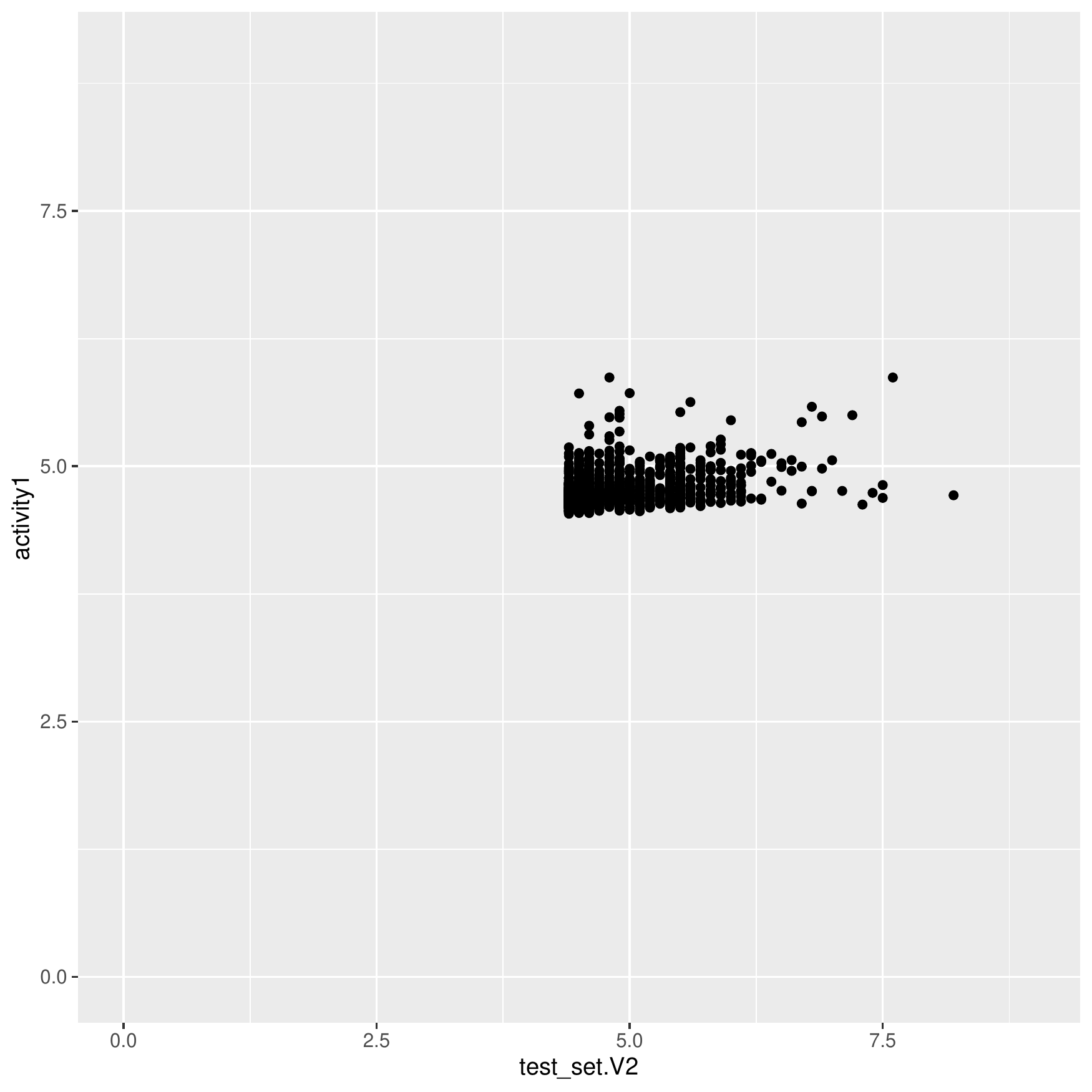}
    \end{subfigure}
    \caption{Actual values compared to predicted values for training set (left) and test set (right).\tb{I would szggest to focus on the range above 4, i.e., rescale the axes, to show more detail. Also, axis labels in the figures need to be improved, make them more clear and related to terminology used in the text. Moreover, some quantitative evaluation of the quality of the models is missing, e.g., $R^2$ or AUC or whatever, to compare the methods and evaluate this final result.}}\label{fig:y-y_Plot}
\end{figure}

Since overfitting can be a problem in model building when the model corresponds too closely to the training set, we analysed if this is the case in our ensemble for drug datasets. We used $y-y$-Plots
\tbdel{I know what this means, but why not saying $y$-$\hat{y}$-plots, for clarity?} \cite{emmerich2006single} to plot the actual value and the predicted value as x- and y-axis. The plots for the training and test set are depicted in Figure \ref{fig:y-y_Plot}.







\section{Discussion}\label{ResultsDiscussion}

In summary, the results on the approximation of mathematical benchmark functions and on drug property prediction show that the usage of quadratic programming has three advantages compared to the state-of the art heuristic (1+1)-ES method:
\begin{itemize}
\item it offers a run-time benefit providing a solution almost instantaneously, while the (1+1)-ES took several seconds to converge; 
\item it provides slightly better results regarding model weighted models;
\tb{model weighted models?}
\item it is more reliable, because  "bad surprises" are avoided in cases where the heuristic method might converge prematurely.
\end{itemize}

Moreover, the results on the drug discovery data show that the used of heterogeneous data is possible, and the input does not necessarily be from a continuous domain, as in the previously published results of optimally weighted model mixtures.
Even though the results of the weighted RMSE method is better with the exact approach than with the heuristic approach, it is surprisingly not always the case that it outperforms the non-weighted method that does not compensate for the clustering of data points.


\section{Summary and Outlook}\label{SummaryOutlook}

This research article proposes an exact and efficient convex quadratic programming method on how to find optimally weighted ensembles of models, using convex combinations. The approach is elegant, yet powerful and is shown to yield better results than using only model selection strategies. We also study how we can compensate for overrepresentation of highly clustered regions in the training data and suggested density-based weighting, which can be integrated straightforwardly, and without introducing non-convexity, in the quadratic programming method. 

The results on the approximation of mathematical benchmark functions and on drug property prediction show that the usage of quadratic programming has not only a run-time benefit, but also leads to slightly better results regarding model weights, if compared with heuristic methods. However, apart from mere benchmarking results, the replacement of the heuristic method by a exact convex quadratical programming method has also the advantage that "bad surprises" are avoided in cases where the heuristic method might converge prematurely and thereby we can with the new exact method provide a more reliable approach.

The idea of using weighted training data, so-far, did not yield a strong advantage in the observed test results. But on the other hand the results are also not much worse, and we showed that the complexity of the mathematical programming problem increases only slightly and it remains convex, and thus efficiently solvable. Future work will need to be conducted to gain a better understanding of how the clustering of data points may influence the result of the optimization, and if alternative weighting schemes can be derived that take this effect better into account.

When it comes to applications, it will be beneficial to test the approach on a broader range of problems and machine learning methods. Our first results on heterogeneous data for drug property prediction show first promising result in that direction. 

\bibliographystyle{abbrv}

\end{document}